\documentclass[final]{article}

     \usepackage[nonatbib]{neurips_2020}

\usepackage[utf8]{inputenc} 
\usepackage[T1]{fontenc}    
\usepackage[hidelinks]{hyperref}       
\usepackage{url}            
\usepackage{booktabs}       
\usepackage{amsfonts}       
\usepackage{nicefrac}       
\usepackage{microtype}      

\title{Layer Sparsity in Neural Networks}

\author{Mohamed Hebiri \\
  Department of Mathematics\\
  Universit\'e de Marne-la-Vall\'ee\\
77454 Marne-la-Vall\'ee, France\\
\texttt{mohamed.hebiri@u-pem.fr}\\
\And
  Johannes~Lederer\thanks{www.johanneslederer.com}\\
  Department of Mathematics\\
  Ruhr-University Bochum\\
  44801 Bochum, Germany\\
  \texttt{johannes.lederer@rub.de} \\
}

\usepackage{amsmath, amsthm, amssymb, amsfonts}

\newtheorem{theorem}{Theorem}

\newtheorem{example}{Example}

\usepackage{comment}
\usepackage{xcolor}

\usepackage{tikz}
\usetikzlibrary{positioning}
\tikzset{
  included node/.style={circle, draw=black!100, thick, on grid, minimum width=0.5cm}, 
  hidden node/.style={circle, draw=black!30, thick, on grid, minimum width=0.5cm},
  included connection/.style={->, thick, draw=black!100},
  hidden connection/.style={->, thick, draw=black!30, draw opacity=0},
  fused connection/.style={->, thick, draw=black!100},
  fidden connection/.style={->, thick, draw=black!30, draw opacity=0},
  node title/.style={above=0.8cm of five-one, font=\bfseries},
  general/.style={node distance=1cm and 0.7cm}
}

\newcommand{\plotdistance}{~~~}

\usepackage{pdfsync}

\usepackage{bm}
\usepackage{arydshln} 
\newcommand{\coloring}[1]{\textcolor{red}{#1}}
\renewcommand{\coloring}[1]{#1}

\newcommand{\methodLS}{\ensuremath{\operatorname{LS}}}
\newcommand{\methodSLS}{\ensuremath{\operatorname{SLS}}}
\newcommand{\methodFLS}{\ensuremath{\operatorname{FLS}}}
\newcommand{\methodOLS}{\ensuremath{\operatorname{ILS}}}

\newcommand{\methodSLSB}{\textbf{SLS}}
\newcommand{\methodFLSB}{\textbf{FLS}}

\newcommand{\layernbr}{\coloring{\ensuremath{l}}}
\newcommand{\samplenbr}{\coloring{\ensuremath{n}}}
\newcommand{\parameternbr}{\coloring{\ensuremath{p}}}
\newcommand{\inputnbr}{\coloring{\ensuremath{d}}}

\newcommand{\tuningCS}{\coloring{\ensuremath{r^{\operatorname{C}}}}}
\newcommand{\tuningC}{\coloring{\ensuremath{\boldsymbol{r}^{\operatorname{C}}}}}
\newcommand{\tuningNS}{\coloring{\ensuremath{r^{\operatorname{N}}}}}
\newcommand{\tuningN}{\coloring{\ensuremath{\boldsymbol{r}^{\operatorname{N}}}}}
\newcommand{\tuningLS}{\coloring{\ensuremath{r^{\operatorname{L}}}}}

\newcommand{\tuningL}{\coloring{\ensuremath{\boldsymbol{r}^{\operatorname{L}}}}}

\newcommand{\truth}{\coloring{\ensuremath{\boldsymbol{W}}}}
\newcommand{\truths}{\coloring{\ensuremath{W}}}

\newcommand{\datafit}{\coloring{\ensuremath{\operatorname{DataFit}}}}

\newcommand{\R}{\coloring{\ensuremath{\mathbb{R}}}}

\newcommand{\parameterspace}{\coloring{\ensuremath{\mathcal{V}}}}
\newcommand{\parameterspaceMerged}{\coloring{\ensuremath{\mathcal{V}_{\sparseset}}}}

\newcommand{\estimator}{\coloring{\ensuremath{\widehat{\truth}}}}

\newcommand{\estimators}{\coloring{\ensuremath{\widehat{\truths}}}}

\newcommand{\estimatorMerged}{\coloring{\ensuremath{\widehat{\truth}_{\sparseset}}}}

\newcommand{\estimatorMergeds}{\coloring{\ensuremath{\widehat{\truths}_{\sparseset}}}}

\newcommand{\weights}{\coloring{\ensuremath{V}}}
\newcommand{\weight}{\coloring{\ensuremath{\boldsymbol{V}}}}

\newcommand{\weightsReLU}{\coloring{\ensuremath{N}}}
\newcommand{\weightReLU}{\coloring{\ensuremath{\boldsymbol{N}}}}

\newcommand{\weightsMerged}{\coloring{\ensuremath{M}}}
\newcommand{\weightMerged}{\coloring{\ensuremath{\boldsymbol{M}}}}

\newcommand{\invector}{\coloring{\ensuremath{\boldsymbol{x}}}}

\newcommand{\outvectors}{\coloring{\ensuremath{y}}}
\newcommand{\noises}{\coloring{\ensuremath{u}}}

\newcommand{\actfunction}{\coloring{\ensuremath{\boldsymbol{f}}}}

\newcommand{\regfunction}{\coloring{\ensuremath{f}}}

\newcommand{\regularizer}{\coloring{\ensuremath{h}}}

\newcommand{\regularizerC}{\coloring{\ensuremath{\regularizer^{\operatorname{C}}}}}
\newcommand{\regularizerN}{\coloring{\ensuremath{\regularizer^{\operatorname{N}}}}}
\newcommand{\regularizerL}{\coloring{\ensuremath{\regularizer^{\operatorname{L}}}}}
\newcommand{\regularizerLj}{\coloring{\ensuremath{\regularizer^{\operatorname{L},j}}}}

\newcommand{\norm}[1]{|\!|#1|\!|}
\newcommand{\normtwo}[1]{\norm{#1}_2}

\newcommand{\mnorm}[1]{|\!|\!|#1|\!|\!|}
\newcommand{\mnormone}[1]{\mnorm{#1}_1}
\newcommand{\mnormonetwo}[1]{\mnorm{#1}_{2,1}}
\newcommand{\mnormplustwo}[1]{\mnorm{#1}_{2,+}}

\newcommand{\mnormn}[1]{\mnorm{#1}_n}

\newcommand{\inputM}{\coloring{\ensuremath{\boldsymbol{c}}}}
\newcommand{\inputMS}{\coloring{\ensuremath{c}}}
\newcommand{\inputMM}{\coloring{\ensuremath{\boldsymbol{t}}}}
\newcommand{\inputMMS}{\coloring{\ensuremath{t}}}
\newcommand{\inputMMM}{\coloring{\ensuremath{\boldsymbol{z}}}}

\newcommand{\boundbound}{\coloring{\ensuremath{A}}}
\newcommand{\boundtruth}{\coloring{\ensuremath{B}}}
\newcommand{\boundGauss}{\coloring{\ensuremath{G}}}

\newcommand{\parametertotal}{\coloring{\ensuremath{P}}}

\newcommand{\constT}{\coloring{\ensuremath{a_1}}}
\newcommand{\constTT}{\coloring{\ensuremath{a_2}}}
\newcommand{\constTTT}{\coloring{\ensuremath{a_3}}}

\newcommand{\sparseset}{\coloring{\ensuremath{\mathcal S}}}
\newcommand{\sparsesize}{\coloring{\ensuremath{s}}}

\newcommand{\sparsesizeW}{\coloring{\ensuremath{s_W}}}

\newcommand{\sparsevector}{\coloring{\ensuremath{\boldsymbol s}}}

\newcommand{\sparsevectors}{\coloring{\ensuremath{s}}}

\newcommand{\testset}{\ensuremath{\mathcal{T}}}
\newcommand{\MSE}{\ensuremath{\widehat{\operatorname{mse}}}}
\newcommand{\LayerSpars}{\ensuremath{\hat{\operatorname{s}}}}

\newcommand{\sparsityW}{{\rm s_W}}

\begin{document}

\maketitle

\begin{abstract}
Sparsity has become popular in machine learning,
because it can save computational resources, facilitate interpretations, and prevent overfitting.
In this paper, we discuss sparsity in the framework of neural networks.
In particular, we formulate a new notion of sparsity that concerns the networks' layers and, therefore,
aligns particularly well with the current trend toward deep networks.
We call this notion layer sparsity.
We then introduce  corresponding regularization and refitting schemes
that can complement standard deep-learning pipelines to generate more compact and accurate networks.
\end{abstract}

\section{Introduction}
\label{Introduction}
The number of layers and the number of nodes in each layer are arguably among the most fundamental parameters of neural networks.
But specifying these parameters can be challenging:
deep and wide networks, that is, networks with many layers and nodes,
 can describe data in astounding detail,
but they are also prone to overfitting and require large memory, CPU, energy, and so forth.
 The resource requirements can be particularly problematic for real-time applications or applications on  fitness trackers and other wearables,
whose popularity has surged in recent years.
A promising approach to meet these challenges is to fit  networks sizes adaptively,
that is,
to allow for many layers and nodes in principle,
but to ensure that the final network is ``simple'' in that it has a small number of connections, nodes, or layers~\cite{Changpinyo17,Han15,Kim16,Liu15,Wen16}.

Popular ways to fit such simple and compact networks include successively augmenting small networks~\cite{Ash89, Bello92},
pruning large networks~\cite{Simonyan14},
or explicit sparsity-inducing regularization of the weight matrices,
which we focus on here.
An example is the $\ell_1$-norm,
which can reduce the  number of connections.
Another example is the $\ell_1$-norm grouped over the rows of the weight matrices,
which can reduce the number of nodes.
It has been shown that such  regularizers can indeed produce networks that are both accurate and yet have a small number of nodes and connections either in the first layer \cite{Feng17} or overall \cite{Alvarez16,Liu15,Scardapane17}.
Such sparsity-inducing regularizers also have a long-standing tradition and  thorough theoretical underpinning in  statistics~\cite{Hastie15}.

But while sparsity on the level of connections and nodes has been studied in some detail,
sparsity on the level of layers is much less understood.
This lack of understanding contrasts the current trend to deep network architectures,
which is supported by state-of-the-art performances of deep networks~\cite{LeCun15,Schmidhuber15}, recent approximation theory for ReLU activation networks~\cite{Liang16,Telgarsky16,Yarotsky17}, and  recent statistical theory~\cite{Golowich17,Kohler19,Taheri20}.
Hence,
a better understanding of sparsity  on the level of layers seems to be in order.

Therefore,
we discuss in this paper sparsity with a special emphasis on the networks' layers.
Our key observation is that for typical activation functions such as ReLU,
a layer can be removed if all its parameter values are non-negative
We leverage this observation in the development of a new regularizer that specifically targets sparsity on the level of layers,
and we show that this regularizer can lead to more compact and more accurate networks.

Our three main contributions are:
\begin{enumerate}
    \item We introduce a new notion of sparsity that we call \emph{layer sparsity}.
\item We introduce a corresponding regularizer that can reduce network sizes.
    \item We introduce an additional refitting step  that can  further improve prediction accuracies.
\end{enumerate}

In Section~\ref{Sparsity}, we specify our framework, discuss different notions of sparsity, and introduce our refitting scheme.
In Section~\ref{Numerics},
we establish a numerical proof of concept.
In Section~\ref{Proofs}, we establish the proof for our theoretical result.
In Section~\ref{Discussion}, we conclude with a discussion.

\section{Sparsity in Neural Networks}
\label{Sparsity}

\newcommand{\figurenetworks}[1]{
\begin{figure}[#1]
\centering
\begin{tikzpicture}[general]  
  \node[included node](one-one){};
  \node[included node](one-two)[right=of one-one]{};
  \node[included node](one-three)[right=of one-two]{};
  
  \node[included node](two-one) [above=of one-one]{};
  \node[included node](two-two)[right=of two-one]{};
  \node[included node](two-three)[right=of two-two]{};
  
  \node[included node](three-one) [above=of two-one]{};
  \node[included node](three-two)[right=of three-one]{};
  \node[included node](three-three)[right=of three-two]{};
  
  \node[included node](four-one) [above=of three-one]{};
  \node[included node](four-two)[right=of four-one]{};
  \node[included node](four-three)[right=of four-two]{};

  \node[included node](five-one)[above=of four-two]{};
  
  \draw[included connection] (one-one)--(two-one);
  \draw[included connection] (one-one)--(two-two);
  \draw[included connection] (one-one)--(two-three);
  
  \draw[included connection] (one-two)--(two-one);
  \draw[included connection] (one-two)--(two-two);
  \draw[included connection] (one-two)--(two-three);
  
  \draw[included connection] (one-three)--(two-one);
  \draw[included connection] (one-three)--(two-two);
  \draw[included connection] (one-three)--(two-three);
  
  \draw[included connection] (two-one)--(three-one);
  \draw[included connection] (two-one)--(three-two);
  \draw[included connection] (two-one)--(three-three);
  
  \draw[included connection] (two-two)--(three-one);
  \draw[included connection] (two-two)--(three-two);
  \draw[included connection] (two-two)--(three-three);
  
  \draw[included connection] (two-three)--(three-one);
  \draw[included connection] (two-three)--(three-two);
  \draw[included connection] (two-three)--(three-three);
  
  \draw[included connection] (three-one)--(four-one);
  \draw[included connection] (three-one)--(four-two);
  \draw[included connection] (three-one)--(four-three);
  
  \draw[included connection] (three-two)--(four-one);
  \draw[included connection] (three-two)--(four-two);
  \draw[included connection] (three-two)--(four-three);
  
  \draw[included connection] (three-three)--(four-one);
  \draw[included connection] (three-three)--(four-two);
  \draw[included connection] (three-three)--(four-three);
  
  \draw[included connection] (four-one)--(five-one);
  \draw[included connection] (four-two)--(five-one);
  \draw[included connection] (four-three)--(five-one);
  
  \node[node title] {no sparsity};
  \node[below=0.1cm of one-two] {input};
  \node[above=0.1cm of five-one] {output};
\end{tikzpicture}\plotdistance
\begin{tikzpicture}[general]  
  \node[included node](one-one){};
  \node[included node](one-two)[right=of one-one]{};
  \node[included node](one-three)[right=of one-two]{};
  
  \node[included node](two-one) [above=of one-one]{};
  \node[included node](two-two)[right=of two-one]{};
  \node[included node](two-three)[right=of two-two]{};
  
  \node[included node](three-one) [above=of two-one]{};
  \node[included node](three-two)[right=of three-one]{};
  \node[included node](three-three)[right=of three-two]{};
  
  \node[included node](four-one) [above=of three-one]{};
  \node[included node](four-two)[right=of four-one]{};
  \node[included node](four-three)[right=of four-two]{};

  \node[included node](five-one)[above=of four-two]{};
  
  \draw[included connection] (one-one)--(two-one);
  \draw[included connection] (one-one)--(two-two);
  \draw[hidden connection] (one-one)--(two-three);
  
  \draw[hidden connection] (one-two)--(two-one);
  \draw[included connection] (one-two)--(two-two);
  \draw[hidden connection] (one-two)--(two-three);
  
  \draw[hidden connection] (one-three)--(two-one);
  \draw[included connection] (one-three)--(two-two);
  \draw[included connection] (one-three)--(two-three);
  
  \draw[included connection] (two-one)--(three-one);
  \draw[hidden connection] (two-one)--(three-two);
  \draw[hidden connection] (two-one)--(three-three);
  
  \draw[hidden connection] (two-two)--(three-one);
  \draw[included connection] (two-two)--(three-two);
  \draw[included connection] (two-two)--(three-three);
  
  \draw[hidden connection] (two-three)--(three-one);
  \draw[hidden connection] (two-three)--(three-two);
  \draw[included connection] (two-three)--(three-three);
  
  \draw[included connection] (three-one)--(four-one);
  \draw[hidden connection] (three-one)--(four-two);
  \draw[hidden connection] (three-one)--(four-three);
  
  \draw[hidden connection] (three-two)--(four-one);
  \draw[included connection] (three-two)--(four-two);
  \draw[hidden connection] (three-two)--(four-three);
  
  \draw[included connection] (three-three)--(four-one);
  \draw[included connection] (three-three)--(four-two);
  \draw[included connection] (three-three)--(four-three);
  
  \draw[included connection] (four-one)--(five-one);
  \draw[included connection] (four-two)--(five-one);
  \draw[included connection] (four-three)--(five-one);
  
  \node[node title] {connection sparsity};
  \node[below=0.1cm of one-two] {input};
  \node[above=0.1cm of five-one] {output};
\end{tikzpicture}\plotdistance
\begin{tikzpicture}[general]  
  \node[included node](one-one){};
  \node[included node](one-two)[right=of one-one]{};
  \node[included node](one-three)[right=of one-two]{};
  
  \node[included node](two-one) [above=of one-one]{};
  \node[included node](two-two)[right=of two-one]{};
  \node[hidden node](two-three)[right=of two-two]{};
  
  \node[included node](three-one) [above=of two-one]{};
  \node[included node](three-two)[right=of three-one]{};
  \node[included node](three-three)[right=of three-two]{};
  
  \node[included node](four-one) [above=of three-one]{};
  \node[hidden node](four-two)[right=of four-one]{};
  \node[hidden node](four-three)[right=of four-two]{};

  \node[included node](five-one)[above=of four-two]{};
  
  \draw[included connection] (one-one)--(two-one);
  \draw[included connection] (one-one)--(two-two);
  \draw[hidden connection] (one-one)--(two-three);
  
  \draw[included connection] (one-two)--(two-one);
  \draw[included connection] (one-two)--(two-two);
  \draw[hidden connection] (one-two)--(two-three);
  
  \draw[included connection] (one-three)--(two-one);
  \draw[included connection] (one-three)--(two-two);
  \draw[hidden connection] (one-three)--(two-three);
  
  \draw[included connection] (two-one)--(three-one);
  \draw[included connection] (two-one)--(three-two);
  \draw[included connection] (two-one)--(three-three);
  
  \draw[included connection] (two-two)--(three-one);
  \draw[included connection] (two-two)--(three-two);
  \draw[included connection] (two-two)--(three-three);
  
  \draw[hidden connection] (two-three)--(three-one);
  \draw[hidden connection] (two-three)--(three-two);
  \draw[hidden connection] (two-three)--(three-three);
  
  \draw[included connection] (three-one)--(four-one);
  \draw[hidden connection] (three-one)--(four-two);
  \draw[hidden connection] (three-one)--(four-three);
  
  \draw[included connection] (three-two)--(four-one);
  \draw[hidden connection] (three-two)--(four-two);
  \draw[hidden connection] (three-two)--(four-three);
  
  \draw[included connection] (three-three)--(four-one);
  \draw[hidden connection] (three-three)--(four-two);
  \draw[hidden connection] (three-three)--(four-three);
  
  \draw[included connection] (four-one)--(five-one);
  \draw[hidden connection] (four-two)--(five-one);
  \draw[hidden connection] (four-three)--(five-one);
  
  \node[node title] {node sparsity};
  \node[below=0.1cm of one-two] {input};
  \node[above=0.1cm of five-one] {output};
\end{tikzpicture}\plotdistance
\begin{tikzpicture}[general]  
  \node[included node](one-one){};
  \node[included node](one-two)[right=of one-one]{};
  \node[included node](one-three)[right=of one-two]{};
  
  \node[included node](two-one) [above=of one-one]{};
  \node[included node](two-two)[right=of two-one]{};
  \node[included node](two-three)[right=of two-two]{};
  
  \node[hidden node](three-one) [above=of two-one]{};
  \node[hidden node](three-two)[right=of three-one]{};
  \node[hidden node](three-three)[right=of three-two]{};
  
  \node[included node](four-one) [above=of three-one]{};
  \node[included node](four-two)[right=of four-one]{};
  \node[included node](four-three)[right=of four-two]{};

  \node[included node](five-one)[above=of four-two]{};
  
  \draw[included connection] (one-one)--(two-one);
  \draw[included connection] (one-one)--(two-two);
  \draw[included connection] (one-one)--(two-three);
  
  \draw[included connection] (one-two)--(two-one);
  \draw[included connection] (one-two)--(two-two);
  \draw[included connection] (one-two)--(two-three);
  
  \draw[included connection] (one-three)--(two-one);
  \draw[included connection] (one-three)--(two-two);
  \draw[included connection] (one-three)--(two-three);
  
  \draw[hidden connection] (two-one)--(three-one);
  \draw[hidden connection] (two-one)--(three-two);
  \draw[hidden connection] (two-one)--(three-three);
  
  \draw[hidden connection] (two-two)--(three-one);
  \draw[hidden connection] (two-two)--(three-two);
  \draw[hidden connection] (two-two)--(three-three);
  
  \draw[hidden connection] (two-three)--(three-one);
  \draw[hidden connection] (two-three)--(three-two);
  \draw[hidden connection] (two-three)--(three-three);
  
  \draw[hidden connection] (three-one)--(four-one);
  \draw[hidden connection] (three-one)--(four-two);
  \draw[hidden connection] (three-one)--(four-three);
  
  \draw[hidden connection] (three-two)--(four-one);
  \draw[hidden connection] (three-two)--(four-two);
  \draw[hidden connection] (three-two)--(four-three);
  
  \draw[hidden connection] (three-three)--(four-one);
  \draw[hidden connection] (three-three)--(four-two);
  \draw[hidden connection] (three-three)--(four-three);
  
  \draw[included connection] (four-one)--(five-one);
  \draw[included connection] (four-two)--(five-one);
  \draw[included connection] (four-three)--(five-one);
  
  \draw[fused connection] (two-one)--(four-one);
  \draw[fused connection] (two-two)--(four-one);
  \draw[fused connection] (two-three)--(four-one);
  
  \draw[fused connection] (two-one)--(four-two);
  \draw[fused connection] (two-two)--(four-two);
  \draw[fused connection] (two-three)--(four-two);
  
  \draw[fused connection] (two-one)--(four-three);
  \draw[fused connection] (two-two)--(four-three);
  \draw[fused connection] (two-three)--(four-three);
  
  \node[node title] {layer sparsity};
  \node[below=0.1cm of one-two] {input};
  \node[above=0.1cm of five-one] {output};
\end{tikzpicture}\plotdistance
\begin{tikzpicture}[general]  
  \node[included node](one-one){};
  \node[included node](one-two)[right=of one-one]{};
  \node[included node](one-three)[right=of one-two]{};
  
  \node[included node](two-one) [above=of one-one]{};
  \node[included node](two-two)[right=of two-one]{};
  \node[hidden node](two-three)[right=of two-two]{};
  
  \node[hidden node](three-one) [above=of two-one]{};
  \node[hidden node](three-two)[right=of three-one]{};
  \node[hidden node](three-three)[right=of three-two]{};
  
  \node[included node](four-one) [above=of three-one]{};
  \node[hidden node](four-two)[right=of four-one]{};
  \node[hidden node](four-three)[right=of four-two]{};

  \node[included node](five-one)[above=of four-two]{};
  
  \draw[included connection] (one-one)--(two-one);
  \draw[included connection] (one-one)--(two-two);
  \draw[hidden connection] (one-one)--(two-three);
  
  \draw[hidden connection] (one-two)--(two-one);
  \draw[included connection] (one-two)--(two-two);
  \draw[hidden connection] (one-two)--(two-three);
  
  \draw[hidden connection] (one-three)--(two-one);
  \draw[included connection] (one-three)--(two-two);
  \draw[hidden connection] (one-three)--(two-three);
  
  \draw[hidden connection] (two-one)--(three-one);
  \draw[hidden connection] (two-one)--(three-two);
  \draw[hidden connection] (two-one)--(three-three);
  
  \draw[hidden connection] (two-two)--(three-one);
  \draw[hidden connection] (two-two)--(three-two);
  \draw[hidden connection] (two-two)--(three-three);
  
  \draw[hidden connection] (two-three)--(three-one);
  \draw[hidden connection] (two-three)--(three-two);
  \draw[hidden connection] (two-three)--(three-three);
  
  \draw[hidden connection] (three-one)--(four-one);
  \draw[hidden connection] (three-one)--(four-two);
  \draw[hidden connection] (three-one)--(four-three);
  
  \draw[hidden connection] (three-two)--(four-one);
  \draw[hidden connection] (three-two)--(four-two);
  \draw[hidden connection] (three-two)--(four-three);
  
  \draw[hidden connection] (three-three)--(four-one);
  \draw[hidden connection] (three-three)--(four-two);
  \draw[hidden connection] (three-three)--(four-three);
  
  \draw[included connection] (four-one)--(five-one);
  \draw[hidden connection] (four-two)--(five-one);
  \draw[hidden connection] (four-three)--(five-one);
  
  \draw[fused connection] (two-one)--(four-one);
  \draw[fused connection] (two-two)--(four-one);
  \draw[fidden connection] (two-three)--(four-one);
  
  \draw[fidden connection] (two-one)--(four-two);
  \draw[fidden connection] (two-two)--(four-two);
  \draw[fidden connection] (two-three)--(four-two);
  
  \draw[fidden connection] (two-one)--(four-three);
  \draw[fidden connection] (two-two)--(four-three);
  \draw[fidden connection] (two-three)--(four-three);
  
  \node[node title] {combined sparsity};
  \node[below=0.1cm of one-two] {input};
  \node[above=0.1cm of five-one] {output};
\end{tikzpicture}
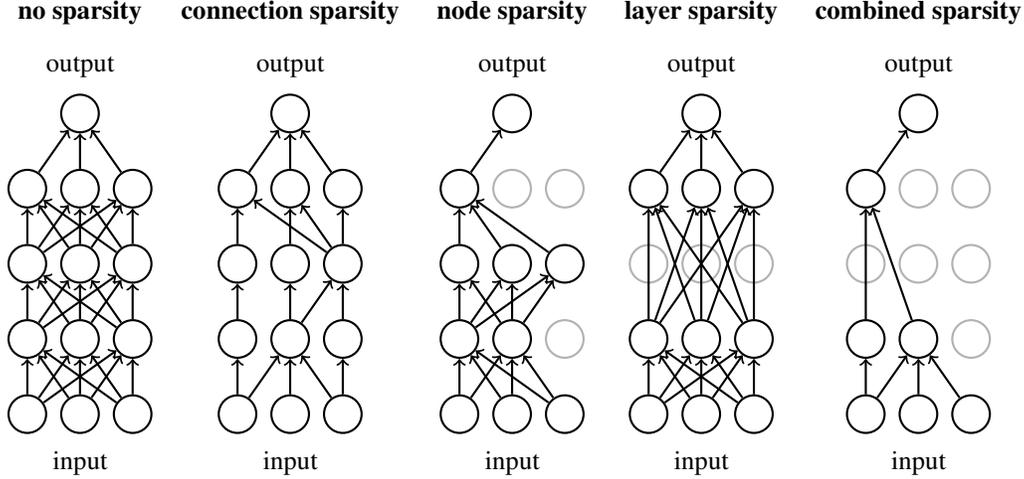
\caption{feedforward neural networks with different types of sparsity}
\label{fig:sparsity}
\end{figure}
}

\newcommand{\figureincline}[1]{
\begin{figure}[#1]
    \centering
    \begin{tikzpicture}[scale=0.5]
    \node at (0, 4.9) {\small\textbf{symmetric $\ell_1$}};
      \draw[->, thick] (-2.4, 0) -- (2.4, 0);
      \draw[->, thick, rotate=90] (-0.4, 0) -- (4, 0);
      \draw (-2, 2) -- (0,0) -- (2,2);
      \draw (-0.1, 2) -- (0.1, 2);
      \draw (2, -0.1) -- (2, 0.1);
      \node at (-0.3, 2) {\tiny 1};
      \node at (2, -0.4) {\tiny 1};
    \end{tikzpicture}~~~~~~~~~~~
    \begin{tikzpicture}[scale=0.5]
    \node at (0, 4.9) {\small\textbf{asymmetric $\ell_1$}};
      \draw[->, thick] (-2.4, 0) -- (2.4, 0);
      \draw[->, thick, rotate=90] (-0.4, 0) -- (4, 0);
      \draw (-2, 4) -- (0,0) -- (2,2);
      \draw (-0.1, 2) -- (0.1, 2);
      \draw (2, -0.1) -- (2, 0.1);
      \node at (-0.3, 2) {\tiny 1};
      \node at (2, -0.4) {\tiny 1};
    \end{tikzpicture}
    \caption{Caption}
    \label{fig:my_label}
\end{figure}
}
We first state our framework,
then discuss different notions of sparsity,
and finally introduce a refitting scheme.

\subsection{Mathematical Framework}

To fix ideas,
we consider feedforward neural networks that model data according to
\begin{equation}\label{model}
    \outvectors_{i}= \actfunction^1\Bigl[\truths^1 \actfunction^2\bigl[... \actfunction^{\layernbr}[\truths^{\layernbr} \invector_i]\bigr] \Bigr]+\noises_i \enspace ,
\end{equation}
where $i\in\{1,\dots,n\}$ indexes the $\samplenbr$~different samples,
$\outvectors_i\in\mathbb{R}$ is the output,
$\invector_i\in\mathbb R^{\inputnbr}$ is the corresponding input with~$\inputnbr$ the input dimension,
$\layernbr$~is the number of layers,
$\truths^j\in\mathbb R^{\parameternbr_{j}\times \parameternbr_{j+1}}$ for $j\in\{1,\dots,\layernbr\}$ are the weight matrices
with $\parameternbr_1=1$ and $\parameternbr_{\layernbr+1}=\inputnbr$,
$\actfunction^j\,:\,\mathbb R^{\parameternbr_j}\to\mathbb R^{\parameternbr_{j}}$ for $j\in\{1,\dots,\layernbr\}$ are the activation functions,
and $\noises_i\in\mathbb R$ is the random noise.

We summarize the parameters in $\truth:=(\truths^1,\dots,\truths^{\layernbr})\in\parameterspace:=\{\weight=(\weights^1,\dots,\weights^{\layernbr}):\weights^j\in\mathbb R^{\parameternbr_{j}\times \parameternbr_{j+1}}\}$,
and we write for ease of notation
\begin{equation}\label{network}
\regfunction_{\weight}[\invector_i]:=\actfunction^1\Bigl[\weights^1 \actfunction^2\bigl[... \actfunction^{\layernbr}[\weights^{\layernbr} \invector_i]\bigr] \Bigr]
\end{equation}
for $\weight\in\parameterspace$.

Neural networks are usually fitted based on regularized estimators in Lagrange
\begin{equation}\label{estimator}
     \estimator\in\operatornamewithlimits{argmin}_{\weight\in\parameterspace}\bigl\{\datafit[\outvectors_1,\dots,\outvectors_{\samplenbr},\invector_1,\dots,\invector_{\samplenbr}]+\regularizer[\weight]\bigr\}
\end{equation}
or constraint form
\begin{equation}\label{estimatorcons}
     \estimator\in\operatornamewithlimits{argmin}_{\substack{\weight\in\parameterspace\\\regularizer[\weight]\leq1}}\bigl\{\datafit[\outvectors_1,\dots,\outvectors_{\samplenbr},\invector_1,\dots,\invector_{\samplenbr}]\bigr\}\,,
\end{equation}
were  $\datafit\,:\,\R^{\samplenbr}\times\R^{\samplenbr\times\inputnbr}$ is a data-fitting function such as least-squares $\sum_{i=1}^{\samplenbr} (\outvectors_i-\regfunction_{\weight}[\invector_i])^2$,
and $\regularizer\,:\,\parameterspace\to[0,\infty)$ is a regularizer such as the elementwise $\ell_1$-norm $\sum_{j,k,l}|(\weights^j)_{kl}|$.
We are particularly interested in regularizers that induce sparsity.

\subsection{Standard Notions of Sparsity}

\figurenetworks{t}

We first state two regularizers that are known in deep learning and the corresponding notions of sparsity.

\paragraph{Connection sparsity}
Consider the vanilla $\ell_1$-regularizer
\begin{equation*}
    \regularizerC[\weight]:=\sum_{j=1}^{\layernbr}(\tuningCS)_j\mnormone{\weights^j}:=\sum_{j=1}^{\layernbr}(\tuningCS)_j\sum_{v=1}^{\parameternbr_j}\sum_{w=1}^{\parameternbr_{j+1}}|(\weights^j)_{vw}|\enspace,
\end{equation*}
where $\tuningC\in[0,\infty)^{\layernbr}$ is a vector of  tuning parameters.
This regularizer is the deep learning equivalent of the lasso regularizer in linear regression~\cite{Tibshirani96}
and has received considerable attention  recently~\cite{Barron2018,Barron2019,Kim16}.
The regularizer acts on each individual connection,
pruning a full network (first network from the left in  Figure~\ref{fig:sparsity}) to a more sparsely connected network (second network in Figure~\ref{fig:sparsity}).
We, therefore, propose to speak of \emph{connection sparsity.}

\paragraph{Node sparsity}
Consider a grouped version of the above regularizer
\begin{equation*}
    \regularizerN[\weight]:=\sum_{j=1}^{\layernbr}(\tuningNS)_j\mnormonetwo{\weights^j}:=\sum_{j=1}^{\layernbr}(\tuningNS)_j\sum_{v=1}^{\parameternbr_j}\sqrt{\sum_{w=1}^{\parameternbr_{j+1}}|(\weights^j)_{vw}|^2}\enspace,
\end{equation*}
where $\tuningN\in[0,\infty)^{\layernbr}$ is again a vector of tuning parameters.
This regularizer is the deep learning equivalent of the group lasso regularizer in linear regression~\cite{Bakin99}
and has received some attention recently~\cite{Alvarez16,Feng17, Scardapane17}.
The regularizer acts on all connections that go into a node simultaneously,
rendering entire nodes inactive (third network in  Figure~\ref{fig:sparsity}).
We, therefore, propose to speak of \emph{node sparsity.}

\subsection{Layer Sparsity}
We now complement the two existing regularizers and notions of sparsity with a new, third notion.

\paragraph{Layer sparsity}
Consider the regularizer 
\begin{equation}
\label{eq:layerregularizer}
    \regularizerL[\weight]:=\sum_{j=1}^{\layernbr-1}(\tuningLS)_j\mnormplustwo{\weights^j}:=\sum_{j=1}^{\layernbr-1}(\tuningLS)_j\sqrt{\sum_{v=1}^{\parameternbr_j}\sum_{w=1}^{\parameternbr_{j+1}}\bigl(\operatorname{neg}[(\weights^j)_{vw}]\bigr)^2}\enspace,
\end{equation}
where
$\tuningL\in[0,\infty)^{\layernbr-1}$ is a vector of tuning parameters, and $\operatorname{neg}[a]:=\min\{a,0\}$  is the negative part of a real value~$a\in\mathbb{R}$.
This regularizers does not have an equivalent in linear regression,
and it is also new in deep learning.

We argue that the regularizer can give rise to a new type of sparsity.
The regularizer can be disentangled along the layers according to
\begin{equation*}
    \regularizerL[\weight]=\sum_{j=1}^{\layernbr-1}(\tuningLS)_j \regularizerLj[\weights^j]
\end{equation*}
with
\begin{equation*}
  \regularizerLj[\weights^j]:=\sqrt{\sum_{v=1}^{\parameternbr_j}\sum_{w=1}^{\parameternbr_{j+1}}\bigl(\operatorname{neg}[(\weights^j)_{vw}]\bigr)^2}~~~~~~~~~\text{for~}j\in\{1,\dots,\layernbr-1\}\enspace.
\end{equation*}
We then focus on an individual layer that corresponds an index~$j\in\{1,\dots,\layernbr-1\}$.
We assume that 1.~the coordinates of the activation functions of the $j$th and $(j+1)$th layers disentangle and 2.~the individual coordinate functions of the $(j+1)$th layer are positive homogeneous.
Precisely,
we assume that for every $\inputMM\in\mathbb{R}^{\parameternbr_j}$, $\inputM\in\mathbb{R}^{\parameternbr_{j+1}}$, and $a\in[0,\infty)$,
it holds that  $\actfunction^{j}[\inputMM]=((\actfunction^{j})_1[\inputMMS_1],\dots,(\actfunction^{{j}})_{\parameternbr_{j}}[\inputMMS_{\parameternbr_{j}}])^\top$,
$\actfunction^{j+1}[\inputM]=((\actfunction^{j+1})_1[\inputMS_1],\dots,(\actfunction^{{j+1}})_{\parameternbr_{j+1}}[\inputMS_{\parameternbr_{j+1}}])^\top$,
and $a(\actfunction^{j+1})_{m}[\inputMS_m]=(\actfunction^{j+1})_{m}[a\inputMS_{m}]$ for all $m\in\{1,\dots,\parameternbr_{j+1}\}$ and some real-valued functions $(\actfunction^{j})_{1},\dots,(\actfunction^{j})_{\parameternbr_{j}},(\actfunction^{j+1})_{1},\dots,(\actfunction^{j+1})_{\parameternbr_{j+1}}$.
Standard examples that fit this framework are ReLU and leaky ReLU networks~\cite{Glorot2011,Hahnloser1998,Hahnloser2000,Salinas1996}.

We can now show that the regularizer~\regularizerL\ induces sparsity on the level of layers.
\begin{theorem}[Layer Sparsity]\label{thm:merging}
Define a merged weight matrix  $\weights^{j,j+1}\in\mathbb{R}^{\parameternbr_{j}\times \parameternbr_{j+2}}$ through $\weights^{j,j+1}:=\parameternbr_{j+1}\weights^{j}\weights^{j+1}$ and a merged  activation function  $\actfunction^{j,j+1}:\mathbb{R}^{\parameternbr_{j}}\to \mathbb{R}^{\parameternbr_{j}}$ through $\actfunction^{j,j+1}:=((\actfunction^{j,j+1})_1,\dots,(\actfunction^{j,j+1})_{\parameternbr_{j}})^\top$ with
\begin{equation*}
    (\actfunction^{j,j+1})_q[\inputMMS]:=(\actfunction^{j})_{q}\Biggl[\frac{1}{\parameternbr_{j+1}}\sum_{m=1}^{\parameternbr_{j+1}}(\actfunction^{j+1})_m[\inputMMS]\Biggr]~~~~~~~~\text{for~}q\in\{1,\dots,\parameternbr_{j}\},\inputMMS\in\mathbb{R} \enspace.
\end{equation*}
It holds that
\begin{equation*}
\regularizerLj[\weights^j]=0~~~~\Rightarrow~~~~\actfunction^{j}\bigl[\weights^{j}\actfunction^{j+1}[\weights^{j+1}\inputMMM]\bigr]=\actfunction^{j,j+1}\bigl[\weights^{j,j+1}\inputMMM\bigr]~~~~~~\text{for all~}\inputMMM\in\mathbb{R}^{\parameternbr_{j+2}}\enspace.
\end{equation*}
\end{theorem}
\noindent That positive homogeneity can allow for moving  weight matrices had  been observed in~\cite{Barron18};
here, we use the positive homogeneity
to merge layers.
The idea is as follows:
$\regularizerLj[\weights^j]=0$ means in view of the stated theorem that we can redefine the network of depth~$\layernbr$ as a network of depth~$\layernbr-1$ by replacing $\actfunction^{j}$ and $\weights^{j}$ by $\actfunction^{j,j+1}$ and $\weights^{j,j+1}$, respectively,
and removing the $(j+1)$th layer.

One can verify readily that
the coordinates of the merged activation functions $\actfunction^{j,j+1}$ can still be disentangled.
Moreover,
the individual coordinate functions of the merged activation functions are positive homogeneous as long as the coordinates of the activation functions that correspond to the $j$th layer are positive homogenous.
Hence, the regularization can merge not only one but many  layers into one.
In conclusion, our new regularizer~$\regularizerL$ acts on all nodes and connections of each layer simultaneously,
rendering entire layers inactive (fourth network in  Figure~\ref{fig:sparsity}).
We, therefore, propose to speak of \emph{layer sparsity.}

The concept of layer sparsity and Theorem~\ref{thm:merging} in particular do not hinge on the exact choice of the regularizer in~\eqref{eq:layerregularizer}:
one can take any function~$\regularizerL$ that can be disentangled along the layers as decribed and that ensures the fact that $\regularizerLj[\weights^j]=0$ implies $\min_{k,l}(\weights^j)_{kl}\geq 0$.

We illustrate layer sparsity with two examples.
\begin{example}[Identity Activation]
We first highlight the meaning of layer sparsity in a simplistic setting.
We consider identity activation, that is, $(\actfunction^j)_q[\inputMMS]=\inputMMS$ for all $j\in\{2,\dots,\layernbr\}$, $q\in\{1,\dots,\parameternbr_j\}$, and $\inputMMS\in\mathbb{R}$.
The networks in~\eqref{network} can then be written as
\begin{equation*}
  \regfunction_{\weight}[\invector_i]=\actfunction^1[\weights^1\cdots\weights^{\layernbr} \invector_i]\enspace .
\end{equation*}
In other words, the initial $\layernbr$-layer network can be compressed into a one-layer network with activation function~$\actfunction^1$ and parameter matrix $\weights^1\cdots\weights^{\layernbr}\in\mathbb{R}^{1\times\inputnbr}$.
This setting with identity activation is, of course, purely academic,
but it motivates an important question:
can parts of networks be compressed similarly in the case of practical activation functions such as ReLU?

Theorem~\ref{thm:merging} gives an answer to this question:
if $\regularizerLj[\weights^j]=0$,
then the $j$th and $(j+1)$th layers can be combined.
In the extreme case $\regularizer^{\operatorname{L},2}[\weights^2],\dots,\regularizer^{\operatorname{L},\layernbr}[\weights^{\layernbr}]=0$,
the network can be condensed into a one-layer network just as in the linear case.
In this sense,
one can understand our layer regularizer is as a measure for the networks' ``distance to linearity.''
\end{example}

\begin{example}[ReLU Activation]
\label{ex:ReLU}
We now illustrate  how layer sparsity compresses and, therefore, simplifies networks  in the case of ReLU activation.
ReLU activation means that $(\actfunction^j)_q[\inputMMS]=\max\{\inputMMS,0\}$ for all $j\in\{2,\dots,\layernbr\}$, $q\in\{1,\dots,\parameternbr_j\}$, and $\inputMMS\in\mathbb{R}$.
One  can verify readily that the merged activation functions of the $j$th and $(j+1)$th layer in Theorem~\ref{thm:merging} are then proportional to the original activation functions of the $j$th layer:
\begin{equation*}
    (\actfunction^{j,j+1})_q[\inputMMS]=(\actfunction^{j})_{q}[\inputMMS]~~~~~~~~\text{for~}q\in\{1,\dots,\parameternbr_{j}\},\inputMMS\in\mathbb{R}\enspace .
\end{equation*}
We show in the following that this leads to concise networks.

We fix an initial network~$\regfunction_{\weightReLU}$ parameterized by~$\weightReLU\in\parameterspace$.
We identify the active layers of the network by
\begin{equation}
\label{eq:activelayersindices}
  \sparseset\equiv\sparseset[\weightReLU]:=\bigl\{j\in\{1,\dots,\layernbr-1\}\,:\,\regularizerLj[\weightsReLU^j]\neq 0\bigr\}\cup\{\layernbr\}\enspace .
\end{equation}
Thus, $\sparseset$ and $\{1,\dots,\layernbr\}\setminus \sparseset$ contain the indexes of the relevant and irrelevant layers, respectively.
The level of sparsity, that is, the number of active layers, is $\sparsesize:=|\sparseset|\leq \layernbr$.

We now denote the indexes in $\sparseset$ in an orderly fashion: $j_1,\dots,j_{\sparsesize}\in \sparseset$ such that $j_1<\dots< j_{\sparsesize}=\layernbr$.
We then define scaled versions of the corresponding merged matrices:
if $j_i-1\in\sparseset$ or $j_i=1$,
we do the ``trivial merge'' $\weightsMerged^{j_i}:=\weightsReLU^{j_i}\in\mathbb{R}^{\parameternbr_{j_i}\times \parameternbr_{j_i+1}}$;
otherwise, we do the ``non-trivial merge''
\begin{equation*}
\weightsMerged^{j_i}:=(\parameternbr_{j_{i-1}+2}\cdots\parameternbr_{j_{i}})\weightsReLU^{j_{i-1}+1}\cdots \weightsReLU^{j_{i}}\in\mathbb{R}^{\parameternbr_{j_{i-1}+1}\times \parameternbr_{j_{i}+1}}\enspace .
\end{equation*}
In other words, we merge all irrelevant layers between the $j_{i-1}$th and $j_{i}$th layers into the $j_i$th layer.

We can then compress the data-generating model in~\eqref{model} into
\begin{equation*}
    \outvectors_{i}= \actfunction^{j_1}\Bigl[\weightsMerged^{j_1} \actfunction^{j_2}\bigl[... \actfunction^{j_s}[\weightsMerged^{j_s} \invector_i]\bigr] \Bigr]+\noises_i
\end{equation*}
with $\weightMerged:=(\weightsMerged^{j_1},\dots,\weightsMerged^{j_s})\in\parameterspaceMerged:=\{\weight=(\weights^1,\dots,\weights^{s}):\weights^i\in\mathbb R^{\parameternbr_{j_{i-1}+1}\times \parameternbr_{j_{i}+1}}\}$.
Formulated differently,
we can condense the original network according to
\begin{equation*}
\regfunction_{\weightReLU}[\invector_i]=\regfunction_{\weightMerged}[\invector_i]=\actfunction^{j_1}\Bigl[ \weightsMerged^{j_1} \actfunction^{j_2}\bigl[... \actfunction^{j_{\sparsesize}}[\weightsMerged^{j_{\sparsesize}} \invector_i]\bigr] \Bigr]\enspace  ,
\end{equation*}
that is, we can formulate the initial ReLU activation network with $\layernbr$~layers as a new ReLU activation network with $\sparsesize$~layers.

The new network is still a ReLU activation network but has a smaller number of layers if $\sparsesize<\layernbr$ and, consequently, a smaller number of parameters in total:
the total number of parameters in the initial network  is $\sum_{j=1}^{\layernbr}(\parameternbr_j\times \parameternbr_{j+1})$,
while the total number of parameters in the transformed network is only $\sum_{i=1}^{\sparsesize}(\parameternbr_{j_{i-1}+1}\times \parameternbr_{j_i+1})$.
\end{example}

Our concept for regularizing layers is substantially different from existing ones:
our layer-wise regularizer induces weights to be \emph{non-negative},
whereas existing layer-wise regularizers  induce weights to be \emph{zero}~\cite[Section~3.3]{Wen16}.
The two main advantages of our approach are that it (i)~does not require shortcuts to avoid trivial networks
and (ii)~does not implicitly enforce connection or node sparsity.
We thus argue that our layer sparsity is a much more natural and appropriate way to capture and regularize network depths.

Layer sparsity more closely relates to ResNets~\cite{He16}.
The recent popularity of ResNets is motivated by two observations:
1.~Solvers seem to struggle with finding good minima of deep networks;
even training accuracies can deteriorate when increasing the number of layers.
 2.~Allowing for linear mappings that short-circut parts of the network seem to help solvers in finding better minima.
From our viewpoint here,
one can argue that ResNets use these linear mappings to regulate network depths adaptively and, therefore, are related to layer sparsity.
But importantly, ResNets are even more complex than the networks they are based on,
while our notion simplifies networks.

Since, as one can verify again readily, all three regularizers are convex,
any combination of them is also convex.
Such combinations can be used to obtain networks that are sparse in two or all three aspects (last network in  Figure~\ref{fig:sparsity}).

\subsection{Refitting}\label{sec:refitting}
Layer-sparse estimators such as~\eqref{estimator} with~$\regularizer=\regularizerL$ provide networks that can be condensed into smaller networks.
 But it is well-known that regularization generates bias.
A question is, therefore, whether we can get the best of both worlds:
can we construct a method that yields smaller networks yet avoids bias?

We propose a pipeline that complements the regularized estimators with an additional step:
after estimating a layer-sparse network,
we condense it and re-adjust its parameters by using an unregularized estimator such as least-squares.
In line with terminology in linear regression~\cite{Lederer13}, 
we call the additional step \emph{refitting}.

While this refitting idea applies very generally,
we formulate a version of it in the framework of Example~\ref{ex:ReLU} to keep the notation light.
\begin{example}[ReLU Activation Cont.]\label{exrefitting}
Consider the model in~\eqref{model} with the specifications of Example~\ref{ex:ReLU},
 and consider a corresponding layer-sparse estimator~\estimator\ of the parameters such as~\eqref{estimator} with~$\regularizer=\regularizerL$.
In line with~\eqref{eq:activelayersindices},
we denote the set of the active layers by
\begin{equation*}
  \sparseset=\sparseset[\estimator]=\bigl\{j\in\{1,\dots,\layernbr-1\}\,:\,\regularizerLj[\estimators^j]\neq 0\bigr\}\cup\{\layernbr\} \enspace
\end{equation*}
and the corresponding parameter space of the condensed network by
$\parameterspaceMerged\equiv\parameterspaceMerged[\estimator]=\{\weight=(\weights^1,\dots,\weights^{s}):\weights^i\in\mathbb R^{\parameternbr_{j_{i-1}+1}\times \parameternbr_{j_{i}+1}}\}$.
The least-squares refitted estimator for the parameters in the condensed network is then
\begin{equation}\label{estimatorRef}
     \estimatorMerged\in\operatornamewithlimits{argmin}_{\weight\in\parameterspaceMerged}\biggl\{\sum_{i=1}^n \bigl(\outvectors_i-\regfunction_{\weight}[\invector_i]\bigr)^2\biggr\}\enspace .
\end{equation}
Hence,
the estimator~\estimator\ complemented with least-squares refitting yields the network
\begin{equation}
\regfunction_{\estimatorMerged}[\invector_i]=\actfunction^{j_1}\Bigl[ \estimatorMergeds^{j_1} \actfunction^{j_2}\bigl[... \actfunction^{j_{\sparsesize}}[\estimatorMergeds^{j_{\sparsesize}} \invector_i]\bigr] \Bigr]\enspace  .
\end{equation}
\end{example}
Of course, the refitting scheme can be adapted to any type of sparsity.
But in the following section,
we focus on layer-sparsity and show that layer-sparse estimator with  refitting can both reduce network sizes and prediction errors simultaneously.

\section{Simulation Study}
\label{Numerics}
We now confirm in a brief simulation study that layer regularization can 1.~can improve prediction accuracies and 2.~reduce the number of active layers.

\subsection{Simulation Framework}
\label{sec:exp:data}
We generate data according to the model in~\eqref{model}.
The most outside activation function~$\actfunction^1$ is the identity function,
 and the coordinates of all other activation functions~$\actfunction^2,\dots,\actfunction^{\layernbr}$ are ReLU functions.
The input vectors~$\invector_1,\dots,\invector_{\samplenbr}$ are jointly independent and standard normally distributed in $\inputnbr$~dimensions;
the noise random variables $\noises_1,\dots,\noises_{\samplenbr}$ are independent of the input, jointly independent, and standard normally distributed in one dimension.
For a given sparsity level~$\sparsesizeW\in[0,1]$,
a vector~$\sparsevector\in\{ 0, 1 \}^{\layernbr - 1}$ with independent Bernoulli distributed entries that have success parameter~$\sparsesizeW$ is generated.
The entries of the parameter matrix~$\truths^1$ are sampled independently from the uniform distribution on~$(-2,2)$,
and the entries of the parameter matrices~$\truths^2,\dots,\truths^{\layernbr}$ are sampled independently from the uniform distribution on~$(0,2)$ if~$\sparsevectors_j=0$ and on~$(-2,2)$ otherwise.
Hence, the parameter~$\sparsesizeW$ controls the level of the layer sparsity:
the smaller~$\sparsesizeW$,
the higher the network's layer sparsity.

In concrete numbers,
the input dimension is~$\inputnbr=2$,
the  network widths are~$\parameternbr_{2}=\dots=\parameternbr_{\layernbr}=5$,
the number of hidden layers is $\layernbr-1\in\{10,25\}$,
and the sparsity level is~$\sparsesize\in\{0.1,0.3,0.9\}$.
Our settings and values represent, of course, only a very small part of possible networks in practice,
but given the generality of our concepts,
any attempt of an exhaustive simulation study must fail,
and the simulations at least allow us (i)~to corroborate our theoretical insights and (ii)~to indicate that our concepts can be very useful in practice.

Datasets of $150$ samples are generated;
$\samplenbr=100$~of the samples are assigned to training and the rest to testing.
The relevant measures for an estimate~\estimator\ of the network's parameters are the empirical mean squared error
\begin{equation*}
\MSE\equiv\MSE[\estimator]:=\frac{1}{|\testset|} \sum_{(\outvectors,\invector) \in \testset} \bigl(\outvectors-\regfunction_{\estimator}[\invector] \bigr)^2 \enspace ,
\end{equation*}
over the test set~$\testset$ with cardinality~$|\testset|=50$
and the level of sparsity among the hidden layers
\begin{equation*}
\LayerSpars\equiv\LayerSpars[\estimator]:=\bigl|\bigl\{j\in\{1,\dots,\layernbr-1\}\,:\,\regularizerLj[\estimators^j]\neq 0\bigr\}\bigr|\enspace .
\end{equation*}
Reported are the medians (and third quantiles in paranthesis) over~$30$ simulation runs for each setting.

\subsection{Methods}
\label{sec:exp:methods}

Our first method ($\methodSLS$) is a standard least-squares complemented with the layer regularizer~\eqref{eq:layerregularizer} in Lagrange form~\estimator.
The baseline for this estimator is vanilla least-squares ($\methodLS$).
Since our estimator---in contrast to least-squares---allows for merging layers,
we can also complement it with our refitting scheme of Section~\ref{sec:refitting} (\methodFLS).
The baseline for our refitted estimator is the least-squares estimator that ``knows'' the relevant layers beforehand ($\methodOLS$),
that is, a least-squares on the relevant layers $\parameterspaceMerged[\truth]$ with~$\truth$ the true parameter---see Example~\ref{ex:ReLU}.
The latter estimator cannot be used in practice,
but it can serve as a benchmark here in the simulations.

The objective functions are optimized by using mini-batch gradient descent with batch size~$10$,
learning rate~$10^{-2}$,
and number of epochs $200$ (for $\layernbr = 10$, $\sparsityW = 0.1,\, 0.3$),
$300$ (for $\layernbr = 10$, $\sparsityW = 0.9$), $400$ (for $\layernbr = 25$, $\sparsityW = 0.1$), and $500$ (otherwise).
The tuning parameters~$(\tuningLS)_j$ are $0.2$ (for $\layernbr = 10$, $\sparsityW = 0.1$),
$0.12$ (for $\layernbr = 10$, $\sparsityW = 0.3$),
$0.07$ (for $\layernbr = 10$, $\sparsityW = 0.9$),
and
$0.05$ (for $\layernbr = 25$).

\newcommand{\msespacing}{~~~~}
\begin{table*}[t!]
\scriptsize
\centering
\setlength{\tabcolsep}{0.02cm}
\begin{tabular}{| l |  c  c |  c  c | c  c | c  c | c  c |}
\hline
&\multicolumn{10}{c|}{~~~~~~~~~~~~$\inputnbr=2$, $\samplenbr=100$}\\
\hline
& \multicolumn{6}{c|}{$\layernbr-1=10$} & \multicolumn{4}{c|}{$\layernbr-1=25$}  \\
\hline
& \multicolumn{2}{c|}{$\sparsityW = 0.1$}
& \multicolumn{2}{c|}{$\sparsityW = 0.3$ }
& \multicolumn{2}{c|}{$\sparsityW = 0.9$}
& \multicolumn{2}{c|}{$\sparsityW = 0.1$}
& \multicolumn{2}{c|}{$\sparsityW = 0.3$}\\
\hline
Method
& \multicolumn{1}{c}{$\MSE$} &  \multicolumn{1}{c|}{$\LayerSpars$}
& \multicolumn{1}{c}{$\MSE$} &  \multicolumn{1}{c|}{$\LayerSpars$}
& \multicolumn{1}{c}{$\MSE$} &  \multicolumn{1}{c|}{$\LayerSpars$}
& \multicolumn{1}{c}{$\MSE$} &  \multicolumn{1}{c|}{$\LayerSpars$}
& \multicolumn{1}{c}{$\MSE$} &  \multicolumn{1}{c|}{$\LayerSpars$}\\
\hline
\methodLS
& $ 1.123 \, (1.396)$\msespacing& ---
& $ 1.112 \, (1.448)  $\msespacing & ---
& $ 1.060 \, (1.349)  $\msespacing & ---
& $ 0.989 \, (1.263)  $\msespacing & ---
& $ 0.988 \, (1.260)  $\msespacing & ---\\
\methodSLSB
& ${ 0.018 \, (0.023) }$\msespacing &  ---
& ${ 0.035 \, (0.616)  }$\msespacing & ---
& ${ 0.203 \, (0.830)  }$\msespacing & ---
& ${ 0.142 \, (1.020)  }$\msespacing & ---
& ${ 0.147 \, (0.985)  }$\msespacing & ---\\
\rule{0pt}{2.5ex}\methodOLS
& $ 0.006 \, (0.011)  $\msespacing & $ 1 \, (2) $
& $ 0.008 \, (0.023)  $\msespacing & $ 3 \, (4) $
& $ 0.907 \, (1.272)  $\msespacing & $ 9 \, (10) $
& $ 0.003 \, (0.024)  $\msespacing & $ 2 \, (4) $
& $ 0.050 \, (1.012)  $\msespacing & $ 7 \, (9)~~ $ \\
\methodFLSB
& ${ 0.007 \, (0.013)  }$\msespacing &  $ 1 \, (1) $
& ${ 0.008 \, (0.128)  }$\msespacing &  $ 2 \, (4) $
& ${ 0.364 \, (1.137)  }$\msespacing &  $ 8 \, (9)~~\hspace{0.2mm} $
& ${ 0.009 \, (0.985)  }$\msespacing &  $ 2 \, (5) $
& ${ 0.021 \, (1.120)  }$\msespacing &  $ 4 \, (13) $ \\
\hline
\end{tabular}
\caption{encouraging layer sparsity can reduce the prediction error and the model complexity}
\label{tab:results2}
\end{table*}
%


\subsection{Results}
\label{sec:exp:results}

The numerical results show that our layer-regularized version~\methodSLS\ can improve on the prediction accuracy of the standard least-squares~\methodLS\ considerably ($\MSE$-columns of the first and second rows in Table~\ref{tab:results2}).
The results also show that the refitting in~\methodFLS\ can improve the prediction accuracy further, and that refitted estimator can rival the infeasible \methodOLS\ in terms of prediction ($\MSE$-columns of the third and fourth rows).
The results finally show that the layer regularization can detect the correct number of layers ($\LayerSpars$-columns of the third and fourth rows).
In summary,
our layer-regularized estimator outmatches the standard least-squares,
and the refitted version of our estimator rivals the infeasible least-squares that knows which are the relevant layers beforehand---both in terms of prediction accuracy and sparsity.
Hence,
layer regularization can condense networks effectively.

The results also reveal that the prediction accuracies of our estimators increase with~$\sparsityW$ decreasing ($\MSE$-columns across different~$\sparsityW$).
This trend is expected:
the higher the layer sparsity,
the more layer-regularization can condense the networks.
This behavior is confirmed in the sparsities ($\LayerSpars$-columns across different~$\sparsityW$).
In other words,
the layer regularization is adaptive to the true layer sparsity.

The tuning parameters~$(\tuningLS)_j$ have been calibrated very roughly by hand.
We expect that a more careful calibration of~\tuningL\ based on cross-validation, for example, accentuates the positive effects of layer sparsity even further.
But since our goal in this section is a general proof of concept for layer sparsity rather than the optimization of a specific deep learning pipeline,
we do not pursue this further here.

The tuning parameters of the descent algorithm, such as the batch size, number of epochs, learning rate, and so forth, have also been calibrated very roughly by hand.
An observation is the fact that 
 all methods can sometimes provide accurate prediction if the number of epochs is extremely large,
but our layer-regularized methods~\methodSLS\ and~\methodFLS\ generally lead to accurate prediction after much less epochs than their unregularized counterpart~\methodLS.
This observation indicates that layer regularization also impacts the algorithmic aspects of deep learning beneficially.

\section{Proofs}
\label{Proofs}
We now provide  a proof for the theorem.
\begin{proof}[Proof of Theorem~\ref{thm:merging}]
A main prerequisite of the proof is that positive weights can be pulled insight activation functions that are positive homogeneous.

We first show that $\regularizerLj[\weights^j]=0$ allows us to pull the weight matrix $\weights^{j}$ inside the activation function~$\actfunction^{j+1}$:
If $\regularizerLj[\weights^j]=0$,
then $(\weights^{j})_{qm}\geq 0$ for all $q\in\{1,\dots,\parameternbr_{j}\}$, $m\in\{1,\dots,\parameternbr_{j+1}\}$.
Hence, we find 
for all $\inputM\in\mathbb{R}^{\parameternbr_{j+1}}$ and $q\in\{1,\dots, \parameternbr_{j}\}$ that
\begin{align*}
    \bigl(\weights^{j}\actfunction^{j+1}[\inputM]\bigr)_q
    &=\sum_{m=1}^{\parameternbr_{j+1}}(\weights^{j})_{qm}(\actfunction^{j+1})_m[\inputMS_m]\\
    &=\sum_{m=1}^{\parameternbr_{j+1}}(\actfunction^{j+1})_m\bigl[(\weights^{j})_{qm}\inputMS_m\bigr].
\end{align*}
We have used the assumed positive homogeneity of~$\actfunction^{j+1}$ in the second line.

Using this result with $\inputM=\weights^{j+1}\inputMMM$ then yields
\begingroup
\allowdisplaybreaks
\begin{align*}
    &(\actfunction^j)_{q}\Bigl[\bigl(\weights^j\actfunction^{j+1}[\weights^{j+1}\inputMMM]\bigr)_q\Bigr]\\
    &=(\actfunction^j)_{q}\Biggl[\sum_{m=1}^{\parameternbr_{j+1}}(\actfunction^{j+1})_m\bigl[(\weights^j)_{qm}(\weights^{j+1}\inputMMM)_m\bigr]\Biggr]\\
    &=(\actfunction^j)_{q}\Biggl[\sum_{m=1}^{\parameternbr_{j+1}}(\actfunction^{j+1})_m\bigl[\bigl\langle(\weights^j)_{q.},\weights^{j+1}\inputMMM\bigr\rangle\bigr]\Biggr]\\
    &=(\actfunction^j)_{q}\Biggl[\sum_{m=1}^{\parameternbr_{j+1}}(\actfunction^{j+1})_m\bigl[\bigl\langle(\weights^{j+1})^\top(\weights^j)_{q.},\inputMMM\bigr\rangle\bigr]\Biggr]\\
    &=(\actfunction^j)_{q}\Biggl[\sum_{m=1}^{\parameternbr_{j+1}}(\actfunction^{j+1})_m\bigl[\bigl\langle(\weights^{j+1})^\top((\weights^j)^\top)_{.q},\inputMMM\bigr\rangle\bigr]\Biggr]\\
    &=(\actfunction^j)_{q}\Biggl[\sum_{m=1}^{\parameternbr_{j+1}}(\actfunction^{j+1})_m\bigl[\bigl\langle\bigl((\weights^{j+1})^\top(\weights^j)^\top\bigr)_{.q},\inputMMM\bigr\rangle\bigr]\Biggr]\\
    &=(\actfunction^j)_{q}\Biggl[\sum_{m=1}^{\parameternbr_{j+1}}(\actfunction^{j+1})_m\bigl[\bigl\langle\bigl((\weights^{j}\weights^{j+1})^\top\bigr)_{.q},\inputMMM\bigr\rangle\bigr]\Biggr]\\
    &=(\actfunction^j)_{q}\Biggl[\sum_{m=1}^{\parameternbr_{j+1}}(\actfunction^{j+1})_m\bigl[\bigl\langle(\weights^{j}\weights^{j+1})_{q.},\inputMMM\bigr\rangle\bigr]\Biggr]\\
    &=(\actfunction^j)_{q}\Biggl[\sum_{m=1}^{\parameternbr_{j+1}}(\actfunction^{j+1})_m\bigl[(\weights^{j}\weights^{j+1}\inputMMM)_q\bigr]\Biggr].
\end{align*}
\endgroup
Using the positive homogeneity of $\actfunction^{j,j+1}$ once more then gives
\begin{equation*}
(\actfunction^j)_{q}\Bigl[\bigl(\weights^j\actfunction^{j+1}[\weights^{j+1}\inputMMM]\bigr)_q\Bigr]    =(\actfunction^j)_{q}\Biggl[\frac{1}{\parameternbr_{j+1}}\sum_{m=1}^{\parameternbr_{j+1}}(\actfunction^{j+1})_m\bigl[(\parameternbr_{j+1}\weights^{j}\weights^{j+1}\inputMMM)_q\bigr]\Biggr].
\end{equation*}
Together with the definitions of  $\weights^{j,j+1}$ and $\actfunction^{j,j+1}$, 
this yields
\begin{equation*}
    (\actfunction^j)_{q}\Bigl[\bigl(\weights^j\actfunction^{j+1}[\weights^{j+1}\inputMMM]\bigr)_q\Bigr]=(\actfunction^j)_{q}\Biggl[\frac{1}{\parameternbr_{j+1}}\sum_{m=1}^{\parameternbr_{j+1}}(\actfunction^{j+1})_m\bigl[(\weights^{j,j+1}\inputMMM)_q\bigr]\Biggr]=(\actfunction^{j,j+1})_{q}\bigl[(\weights^{j,j+1}\inputMMM)_q\bigr].
\end{equation*}
This identity concludes the proof in view of  the assumed forms of~$\actfunction^{j}$ and $\actfunction^{j,j+1}$.
\end{proof}

\section{Discussion}
\label{Discussion}

We have shown that layer sparsity can compress layered networks effectively both in theory (Section~\ref{Sparsity}) and practice (Section~\ref{Numerics}).

Related concepts such as ResNets add complexity to the network descriptions and, therefore, are very hard to analyze statistically.
Layer sparsity, in contrast, simplify network architectures and seem amenable to statistical analyses via recent techniques for regularized deep learning~\cite{Taheri20}.
A statistical treatment of layer sparsity, therefore, seems a feasible topic for further research.

Another topic for further research is the  practical calibration of the tuning parameters.
Tuning parameter calibration is an active topic in machine learning;
thus, we expect that one can use  ideas developed in other frameworks, such as~\cite{Chichignoud16}.


In summary, layer sparsity complements other notions of sparsity that concern individual connections or nodes.
All of these concepts can help to fit networks that are efficient in terms of memory and computations and easy to interpret.




{\small
\bibliographystyle{plain}
\bibliography{Literature}

\begin{thebibliography}{10}

\bibitem{Alvarez16}
J.~Alvarez and M.~Salzmann.
\newblock Learning the number of neurons in deep networks.
\newblock In {\em Adv.\@ Neural Inf.\@ Process Syst.}, pages 2270--2278, 2016.

\bibitem{Ash89}
T.~Ash.
\newblock Dynamic node creation in backpropagation networks.
\newblock {\em Connect.\@ Sci.\@}, 1(4):365--375, 1989.

\bibitem{Bakin99}
S.~Bakin.
\newblock {\em Adaptive regression and model selection in data mining
  problems}.
\newblock PhD thesis, The {A}ustralian {N}ational {U}niversity, 1999.

\bibitem{Barron2018}
A.~Barron and J.~Klusowski.
\newblock Approximation and estimation for high-dimensional deep learning
  networks.
\newblock {\em arXiv:1809.03090}, 2018.

\bibitem{Barron18}
A.~Barron and J.~Klusowski.
\newblock Approximation and estimation for high-dimensional deep learning
  networks.
\newblock {\em arXiv:1809.03090}, 2018.

\bibitem{Barron2019}
A.~Barron and J.~Klusowski.
\newblock Complexity, statistical risk, and metric entropy of deep nets using
  total path variation.
\newblock {\em arXiv:1902.00800}, 2019.

\bibitem{Bello92}
M.~Bello.
\newblock Enhanced training algorithms, and integrated training/architecture
  selection for multilayer perceptron networks.
\newblock {\em IEEE Trans.\@ Neural Netw.}, 3(6):864--875, 1992.

\bibitem{Changpinyo17}
S.~Changpinyo, M.~Sandler, and A.~Zhmoginov.
\newblock The power of sparsity in convolutional neural networks.
\newblock {\em arXiv:1702.06257}, 2017.

\bibitem{Chichignoud16}
M.~Chichignoud, J.~Lederer, and M.~Wainwright.
\newblock A practical scheme and fast algorithm to tune the lasso with
  optimality guarantees.
\newblock {\em J.\@ Mach.\@ Learn.\@ Res.}, 17(1):1--20, 2016.

\bibitem{Feng17}
J.~Feng and N.~Simon.
\newblock Sparse-input neural networks for high-dimensional nonparametric
  regression and classification.
\newblock {\em arXiv:1711.07592}, 2017.

\bibitem{Glorot2011}
X.~Glorot, A.~Bordes, and Y.~Bengio.
\newblock Deep sparse rectifier neural networks.
\newblock In {\em International Conference on Artificial Intelligence and
  Statistics}, volume~15 of {\em Proc.\@ Mach.\@ Learn.\@ Res.}, pages
  315--323, 2011.

\bibitem{Golowich17}
N.~Golowich, A.~Rakhlin, and O.~Shamir.
\newblock Size-independent sample complexity of neural networks.
\newblock {\em arXiv:1712.06541}, 2017.

\bibitem{Hahnloser1998}
R.~Hahnloser.
\newblock On the piecewise analysis of networks of linear threshold neurons.
\newblock {\em Neural Networks}, 11(14):691--697, 1998.

\bibitem{Hahnloser2000}
R.~Hahnloser, R.~Sarpeshkar, M.~Mahowald, R.~Douglas, and H.~Seung.
\newblock Digital selection and analogue amplification coexist in a
  cortex-inspired silicon circuit.
\newblock {\em Nature}, 405:947--951, 2000.

\bibitem{Han15}
S.~Han, H.~Mao, and W.~Dally.
\newblock Deep compression: Compressing deep neural networks with pruning,
  trained quantization and huffman coding.
\newblock In {\em International Conference on Learning Representations}, 2016.

\bibitem{Hastie15}
T.~Hastie, R.~Tibshirani, and M.~Wainwright.
\newblock {\em Statistical learning with sparsity: The lasso and
  generalizations}.
\newblock CRC press, 2015.

\bibitem{He16}
K.~He, X.~Zhang, S.~Ren, and J.~Sun.
\newblock Deep residual learning for image recognition.
\newblock In {\em IEEE Int.\@ Conf.\@ Comput.\@ Vis.\@ Pattern Recognit.},
  pages 770--778, 2016.

\bibitem{Kim16}
J.~Kim, V.~Calhoun, E.~Shim, and J.-H. Lee.
\newblock Deep neural network with weight sparsity control and pre-training
  extracts hierarchical features and enhances classification performance:
  Evidence from whole-brain resting-state functional connectivity patterns of
  schizophrenia.
\newblock {\em Neuroimage}, 124:127--146, 2016.

\bibitem{Kohler19}
M.~Kohler and S.~Langer.
\newblock On the rate of convergence of fully connected very deep neural
  network regression estimates.
\newblock {\em arXiv:1908.11133}, 2019.

\bibitem{LeCun15}
Y.~LeCun, Y.~Bengio, and G.~Hinton.
\newblock Deep learning.
\newblock {\em Nature}, 521:436--444, 2015.

\bibitem{Lederer13}
J.~Lederer.
\newblock Trust, but verify: benefits and pitfalls of least-squares refitting
  in high dimensions.
\newblock {\em arXiv:1306.0113}, 2013.

\bibitem{Liang16}
S.~Liang and R.~Srikant.
\newblock Why deep neural networks for function approximation?
\newblock {\em arXiv:1610.04161}, 2016.

\bibitem{Liu15}
B.~Liu, M.~Wang, H.~Foroosh, M.~Tappen, and M.~Pensky.
\newblock Sparse convolutional neural networks.
\newblock In {\em IEEE Int.\@ Conf.\@ Comput.\@ Vis.\@ Pattern Recognit.},
  pages 806--814, 2015.

\bibitem{Salinas1996}
E.~Salinas and L.~Abbott.
\newblock A model of multiplicative neural responses in parietal cortex.
\newblock {\em Proc.\@ Natl.\@ Acad.\@ Sci.\@ USA}, 93(21):11956--11961, 1996.

\bibitem{Scardapane17}
S.~Scardapane, D.~Comminiello, A.~Hussain, and A.~Uncini.
\newblock Group sparse regularization for deep neural networks.
\newblock {\em Neurocomputing}, 241:81--89, 2017.

\bibitem{Schmidhuber15}
J.~Schmidhuber.
\newblock Deep learning in neural networks: An overview.
\newblock {\em Neural Networks}, 61:85--117, 2015.

\bibitem{Simonyan14}
K.~Simonyan and A.~Zisserman.
\newblock Very deep convolutional networks for large-scale image recognition.
\newblock In {\em International Conference on Learning Representations}, 2015.

\bibitem{Taheri20}
M.~Taheri, F.~Xie, and J.~Lederer.
\newblock Statistical guarantees for regularized networks.
\newblock {\em arXiv:2006.00294}, 2020.

\bibitem{Telgarsky16}
M.~Telgarsky.
\newblock Benefits of depth in neural networks.
\newblock In {\em Annual Conference on Learning Theory}, volume~49 of {\em
  Proc.\@ Mach.\@ Learn.\@ Res.\@}, pages 1517--1539, 2016.

\bibitem{Tibshirani96}
R.~Tibshirani.
\newblock Regression shrinkage and selection via the lasso.
\newblock {\em J.\@ R.\@ Stat.\@ Soc.\@ Ser.\@ B.\@ Stat.\@ Methodol.},
  58(1):267--288, 1996.

\bibitem{Wen16}
W.~Wen, C.~Wu, Y.~Wang, Y.~Chen, and H.~Li.
\newblock Learning structured sparsity in deep neural networks.
\newblock In {\em Adv.\@ Neural Inf.\@ Process Syst.}, pages 2082--2090, 2016.

\bibitem{Yarotsky17}
D.~Yarotsky.
\newblock Error bounds for approximations with deep {ReLU} networks.
\newblock {\em Neural Networks}, 94:103--114, 2017.

\end{thebibliography}
}

\end{document}